\documentclass{article}
\usepackage{nips11submit_e,times}
\usepackage{graphicx} 
\usepackage{subfigure} 
\usepackage{algorithm}
\usepackage{algorithmic}
\usepackage{times}
\usepackage{amsmath,amsthm,amsfonts,amssymb,mathrsfs}

\usepackage[fleqn,tbtags]{mathtools}

\DeclareMathOperator\range{range}

\newcommand{\beq}{\begin{equation}} \newcommand{\eeq}{\end{equation}}

\newcommand{\La}{{\Lambda}}
\newcommand{\la}{{\langle}}
\newcommand{\ra}{{\rangle}}
\newcommand{\ka}{{{\kappa}}}

\newcommand{\bi}{\begin{itemize}}
\newcommand{\be}{\begin{enumerate}}
\newcommand{\ei}{\end{itemize}}
\newcommand{\ee}{\end{enumerate}}

\newcommand{\reg}{{\varphi}}

\newcommand{\R}{{\mathbb R}}
\newcommand{\N}{{\mathbb N}}

\newcommand{\calC}{{\cal C}}

\newcommand{\lam}{{\lambda}}

\newcommand{\calB}{{\cal B}}

\newcommand{\lb}{{\langle}}
\newcommand{\rb}{{\rangle}}

\def\boldf#1{\hbox{\rlap{$#1$}\kern.4pt{$#1$}}}

\newcommand{\trans}{^{\scriptscriptstyle \top}}

\newcommand{\hbeta}{{\hat \beta}}

\newcommand{\bh}{{\hat \beta}}
\newcommand{\lh}{{\hat \lambda}}
\newcommand{\ph}{{\hat p}}
\newcommand{\qh}{{\hat q}}

\mathtoolsset{showonlyrefs,showmanualtags}

\newcommand{\rd}{\R^d}

\DeclareMathOperator\prox{prox}
\DeclareMathOperator\argmin{argmin}
\DeclareMathOperator\argmax{argmax}

\begin{document}

\title{A General Framework for Structured Sparsity via Proximal Optimization}


\nipsfinalcopy
\author{
{\bf Andreas Argyriou} \\
Toyota Technological Institute \\ at Chicago
\and
{\bf Luca Baldassarre} \\
University College London
\and
{\bf Jean Morales} \\
University College London
\and
{\bf Massimiliano Pontil}\\
University College London
}

\maketitle

\begin{abstract}
We study a generalized framework for {\em structured sparsity}. 
It extends the well-known methods of Lasso and
Group Lasso by incorporating additional constraints on
the variables as part of a convex optimization problem.
This framework provides a straightforward way of favouring prescribed {\em sparsity patterns},
such as orderings, contiguous regions and overlapping groups, among others. 
Existing optimization methods are limited to specific constraint sets and tend to not scale well
with sample size and dimensionality. 
We propose a novel {\em first order proximal method}, which builds upon 
results on fixed points and successive approximations.
The algorithm can be applied to a general class of conic and norm constraints sets and relies on a proximity operator subproblem which can be computed explicitly. 
Experiments on different regression problems demonstrate the
efficiency of the optimization algorithm and its scalability with the
size of the problem.  They also demonstrate state of the art
statistical performance, which improves over Lasso and StructOMP.
\end{abstract}

\newcommand{\bproof}{\begin{proof}}
\newcommand{\eproof}{\end{proof}}
\newenvironment{essential}{\begin{minipage}[t]{11.4cm}
\textwidth=15.2cm\parskip=5mm\vspace{0.2cm}}{\vspace{0.5cm}\end{minipage}}

\renewcommand{\theequation}{\thesection.\arabic{equation}}
\numberwithin{equation}{section}
\renewcommand{\bar}{\overline}
\newtheorem{theorem}{Theorem}[section]
\newtheorem{question}{Question}[section]
\newtheorem{proposition}{Proposition}[section]
\newtheorem{lemma}{Lemma}[section]
\newtheorem{corollary}{Corollary}[section]
\newtheorem{definition}{Definition}[section]
\newtheorem{problem}{\em Problem}[section]
\newtheorem{remark}{Remark}[section]
\newtheorem{example}{Example}[section]
\newtheorem{case}{Case}[section]
\newtheorem{assumption}{Assumption}[section]

\renewcommand\proofname{\bf Proof} 
\def\eop{$\rule{1.3ex}{1.3ex}$}
\renewcommand\qedsymbol\eop  
\numberwithin{equation}{section}
\makeatletter



\section{Introduction}


We study the problem of learning a sparse linear regression model. The
goal is to estimate a parameter vector $\beta^* \in \R^n$ from a
vector of measurements $y \in \R^m$, obtained from the model $y = X
\beta^* + \xi$, where $X$ is an $m \times n$ matrix, which may be
fixed or randomly chosen, and $\xi \in \R^m$ is a vector resulting from
the presence of noise. We are interested in sparse estimation under
additional conditions on the sparsity pattern of $\beta^*$. In other
words, not only we do expect that $\beta^*$ is sparse but also that it
exhibits {\em structured sparsity}, namely certain configurations of
its nonzero components are preferred to others. This problem
arises in several applications, such as regression, image denoising,
background subtraction etc. -- see \cite{Jenatton,huang2009} for a discussion.

In this paper, we build upon the structured sparsity framework
recently proposed by \cite{MMP-10}. It is a regularization method,
formulated as a convex, nonsmooth optimization problem over a vector
of auxiliary parameters. This approach provides a constructive way to
favour certain sparsity patterns of the regression vector $\beta$. 
Specifically, this formulation involves a penalty function 
given by the formula $\Omega(\beta|\Lambda) = \inf \left\{ \frac{1}{2} \sum_{i=1}^n \left(\frac{\beta_i^2}{\lambda_i} + \lambda_i\right): \lambda \in \La\right\}$. 
This function can be interpreted as an extension of a well-known
variational form for the $\ell_1$ norm. 
The convex constraint set $\La$ provides a means to incorporate prior
knowledge on the magnitude of the components of the regression vector.
As we explain in Section \ref{sec:2}, the sparsity pattern of $\beta$
is contained in that of the auxiliary vector $\lam$
at the optimum. Hence, if the set $\La$ allows only for certain
sparsity patterns of $\lambda$, the same property will be
``transferred'' to the regression vector $\beta$.

In this paper, we propose a tractable class of regularizers of the above form 
which extends the examples described in \cite{MMP-10}.
Specifically, we study in detail the cases in which the set $\Lambda$
is defined by norm or conic constraints, combined with a linear map. 
As we shall see, these cases include formulations which can be used for
learning {\em graph sparsity} and {\em hierarchical sparsity}, in the
terminology of \cite{huang2009}. That is, the sparsity pattern of the
vector $\beta^*$ may consist of a few contiguous regions in one or
more dimensions, or may be embedded in a tree structure. 
These sparsity problems arise
in several applications, ranging from functional magnetic resonance
imaging \cite{gramfort2009,xiang2009}, to scene recognition in vision
\cite{harzallah2009}, to multi-task learning \cite{AEP,oboz09} and
to bioinformatics \cite{rapaport2008} -- to mention but a few.

A main limitation of the technique described in \cite{MMP-10} is that in many cases of interest 
the penalty function cannot be easily computed. This makes it difficult to solve 
the associated regularization problem. For example \cite{MMP-10} proposes to use block coordinate descent, and this method is feasible only for a limited choice of the set $\Lambda$.
The main contribution of this paper is an efficient proximal point method to solve regularized 
least squares 
with the penalty function $\Omega(\beta|\Lambda)$ in the general case of set $\Lambda$ described above. The method combines a fast fixed point iterative scheme, which is inspired by recent work by \cite{MSX} with an accelerated first order method equivalent to
FISTA \cite{fista}. 
We present a numerical study of the
efficiency of the proposed method and a statistical comparison of the
proposed penalty functions with the greedy method of \cite{huang2009}
and the Lasso.

Recently, there has been significant research interest on structured sparsity
and the literature on this subject is growing fast, see for example 
\cite{AEP,huang2009,jacob,Jenatton,yuan} and references
therein for an indicative list of papers. In this work, we mainly focus on
convex penalty methods and compare them to greedy methods
\cite{baraniuk2008,huang2009}. The latter provide a natural extension of
techniques proposed in the signal processing community and, as argued
in \cite{huang2009}, allow for a significant performance improvement
over more generic sparsity models such as the Lasso or the
Group Lasso \cite{yuan}. The former methods have until recently
focused mainly on extending the Group Lasso, by
considering the possibility that the groups overlap according to
certain hierarchical structures \cite{Jenatton,binyu}.
Very recently, general choices of convex penalty functions have
been proposed
\cite{submodular,MMP-10}. 
In this paper we build upon \cite{MMP-10}, providing  
both new instances of the penalty function and improved optimization
algorithms.

The paper is organized as follows. In Section \ref{sec:2}, we set our
notation, review the method of \cite{MMP-10} and present the new sets
for inducing structured sparsity.
In Section \ref{sec:prox}, we present our technique for computing the proximity operator 
of the penalty function and the resulting accelerated proximal method.
In Section \ref{sec:5}, we report numerical experiments with this method.

\section{Framework and extensions}
\label{sec:2}

In this section, we introduce our notation, review the learning
method which we extend in this paper and present the new sets for
inducing structured sparsity.

We let $\R_+$ and $\R_{++}$ be the nonnegative and positive real line,
respectively. For every $\beta \in \R^n$ we define $|\beta| \in
\R_+^n$ to be the vector 
$|\beta| = (|\beta_i|)_{i=1}^n$. For every $p \geq 1$, we
define the $\ell_p$ norm of $\beta$ as $\|\beta\|_p = \left(\sum_{i=1}^n
 |\beta_i|^p\right)^{\frac{1}{p}}$.
If $C \subseteq \R^n$, we denote by $\delta_C: \R^n 
\rightarrow \R$ the indicator function of the set $C$, that is, $\delta_C(x) = 0$ if $x 
\in C$ and $\delta_C(x) = +\infty$ otherwise.

We now review the structured sparsity approach of \cite{MMP-10}.
Given an $m \times n$ input data matrix $X$ and an output vector $y \in \R^m$, obtained
from the linear regression model $y=X\beta^*+\xi$ discussed earlier, they consider 
the optimization problem
%
\beq
\inf \left\{\frac{1}{2} \|X\beta-y\|^2_2 +   
\rho\, \Gamma(\beta,\lam)
: \beta \in \R^n, \lambda\in\La
 \right\}
\label{eq:primal}
\eeq
where $\rho$ is a positive parameter, $\La$ is a prescribed convex
subset of the positive orthant $\R^n_{++}$ and the function $\Gamma :
\R^n \times \R_{++}^n \rightarrow \R$ is given by the formula
$$\Gamma(\beta,\lam) = \frac{1}{2}\sum_{i=1}^n
\left(\frac{\beta_i^2}{\lambda_i} + \lambda_i\right).$$ 
Note that the
infimum over $\lambda$ in general is not attained, however the infimum
over $\beta$ is always attained. 
Since the auxiliary vector $\lam$ appears only in the second term
and our goal is to estimate $\beta^*$, we may also directly consider
the regularization problem
\beq
\min \left\{\frac{1}{2}\|X\beta-y\|^2_2 + \rho\, \Omega(\beta|\Lambda) :
\beta \in \R^n\right\}, 
\label{eq:method}
\eeq
where the penalty function takes the form $\Omega(\beta|\Lambda) = \inf
\left\{\Gamma(\beta,\lam): \lam \in \La\right\}$.
This problem is still convex because the function $\Gamma$ is jointly
convex \cite{boyd}. Also, note that the function $\Omega$ is independent of
the signs of the components of $\beta$.

In order to gain some insight about the above methodology, we note that,
for every $\lam \in \R_{++}^n$, the quadratic function 
$\Gamma(\cdot,\lam)$ provides an upper bound to the $\ell_1$ norm, namely
it holds that $\Omega(\beta|\Lambda) \geq \|\beta\|_1$ and the inequality is tight if and only if $|\beta| \in
\La$. 
This fact is an immediate consequence of the
arithmetic-geometric inequality. In particular, we see that if we 
choose $\La = \R_{++}^n$, the method \eqref{eq:method} reduces to the
Lasso\footnote{More generally, method
\eqref{eq:method} includes the Group Lasso method, see \cite{MMP-10}.}. 
The above observation suggests a heuristic interpretation of the 
method \eqref{eq:method}: among all vectors $\beta$ which have a fixed 
value of the $\ell_1$ norm, the penalty function $\Omega$ will
encourage those for which $|\beta| \in \La$. Moreover, when 
$|\beta| \in \La$ the function $\Omega$ reduces to the $\ell_1$ norm and, so, 
the solution of problem $\eqref{eq:method}$ is expected to be
sparse. The penalty function therefore will encourage certain desired
sparsity patterns. 

The last point can be better understood by looking at problem
\eqref{eq:primal}. For every solution $(\hbeta,\lh)$, 
the sparsity pattern of $\hbeta$ is contained in the sparsity pattern 
of $\lh$, that is, the indices associated with nonzero components of $\hbeta$ are a subset of
those of $\lh$. Indeed, if $\lh_i=0$ it must hold that $\hbeta_i = 0$ as well,
since the objective would diverge otherwise (because of the ratio
$\beta_i^2/\lambda_i$). Therefore, if the set $\La$ favours
certain sparse solutions of $\lh$, the same sparsity pattern will be reflected
on $\hbeta$. Moreover, the regularization term $\sum_i \lambda_i$ 
favours sparse vectors $\lambda$.
For example, a constraint of the form $\lam_1\geq \dots \geq \lam_n$, 
favours consecutive zeros at the end of $\lambda$ and nonzeros
everywhere else. This will lead to zeros at the end of $\beta$ as
well. Thus, in many cases like this, it is easy to incorporate a
convex constraint on $\lambda$, whereas it may not be possible to do
the same with $\beta$. 

In this paper we consider sets $\Lambda$ of the form
\[
\Lambda = \{\lambda\in\R^n_{++} : A\lam \in S\}
\label{eq:genLa}
\]
where $S$ is a convex set and $A$ is a $k \times n$ matrix.
Two main choices of interest are when $S$ is a convex cone or the unit ball of a norm. We shall
refer to the corresponding set $\Lambda$ as {\em conic constraint} or {\em norm constraint} set, respectively.
We next discuss two specific examples, which highlight the flexibility
of our approach and help us understand the sparsity patterns favoured by each choice.

Within the conic constraint sets, we may choose $S = \R_{++}^k$, so that
$\La = \{\lam \in\R^n_{++} : A \lam \geq 0 \}$, which can be used to
encourage hierarchical sparsity. In \cite{MMP-10} they considered the
set $\Lambda = \{\lambda \in\R^n_{++}: \lambda_1\geq\dots\geq
\lambda_n\}$ and derived an explicit formula of the corresponding
regularizer $\Omega(\beta|\Lambda)$. Note that for a generic matrix
$A$ the penalty function cannot be computed explicitly. In Section 3,
we show how to overcome this difficulty.

Within the norm constraint sets, we may choose $S$ to be the $\ell_1$-unit ball and $A$ the edge map of a graph $G$ with edge set $E$, so that
$
\La = \left\{\lam \in\R^n_{++} : \sum_{(i,j) \in E} |\lambda_i - \lambda_j| \leq 1\right\}.
$
This set can be used to encourage sparsity patterns consisting of few
connected regions/subgraphs of the graph $G$. For example if $G$ is a
1D-grid we have that $\La = \{\lam \in\R^n_{++} : \sum_{i=1}^{n-1}
|\lambda_{i+1} - \lambda_i| \leq 1\}$, so the corresponding penalty
will favour vectors $\beta$ whose absolute values are constant. 

For a generic convex set $\La$, since the penalty function $\Omega$ is not
easily computable, one needs to deal directly with problem \eqref{eq:primal}. 
To this end, we recall here the definition of the proximity
operator \cite{moreau62}.
\setlength{\abovedisplayskip}{1ex}
\setlength{\belowdisplayskip}{1ex}
\begin{definition}
Let $\reg$ be a real-valued convex function on $\rd$. The proximity
operator of $\reg$ is defined, for every $t\in\rd$ by $\prox_\reg (t) := \argmin \left\{ \dfrac{1}{2} \|z-t\|_2^2 + \reg(z) : z \in \rd \right\}$.
\label{def:prox}
\end{definition}
The proximity operator is well-defined, because the above minimum exists 
and is unique.
\section{Optimization Method}
\label{sec:prox}

In this section, we discuss how to solve problem (2.1) using an accelerated first-order method that scales linearly with respect to the problem size, as we later show in the experiments.
This method relies on the computation of the proximity operator of the function $\Gamma$, restricted to $\R^n \times \Lambda$.
Since the exact computation of the proximity operator is possible only in simple special cases of sets $\Lambda$, we present here an efficient fixed-point algorithm for computing the proximity operator that can be applied to a wide variety of constraints. 
Finally, we discuss an accelerated proximal method that uses our algorithm.



\subsection{Computation of the Proximity Operator}
\label{sec:prox_map}

According to Definition \ref{def:prox}, the proximal operator of $\Gamma$ at $(\alpha,\mu) \in \R^n \times \R^n$ 
is the solution of the problem
\beq
\min
\left\{
\frac{1}{2}
 \|(\beta,\lam) - (\alpha,\mu)\|^2_2 + \rho \,\Gamma(\beta,\lam) : 
\beta \in \R^n, \lam \in \Lambda \right\}. 
\label{eq:yyyy}
\eeq
For fixed $\lam$, a direct computation yields that the objective function in \eqref{eq:yyyy} attains its minimum at 
\beq
\beta_i(\lam) = \frac{\alpha_i \lam_i}{\lam_i + \rho}.
\label{eq:beta-lam}
\eeq
Using this equation we obtain the simplified problem
\beq
\hspace{-.01truecm}
\min\left\{
\frac{1}{2}\|\lam-\mu\|^2 \hspace{-.04truecm}+\hspace{-.04truecm} 
\frac{\rho}{2} \sum_{i=1}^n \hspace{-.1truecm}\left(\hspace{-.08truecm}\frac{\alpha_i^2}{\lam_i+\rho} + \lam_i\hspace{-.08truecm}\right) \hspace{-.03truecm}: \hspace{-.03truecm}\lambda \in \Lambda \right\}.
\hspace{-.08truecm}
\label{eq:sim}
\eeq
This problem can still be interpreted as a proximity map computation. We discuss how to
solve it under 
our general assumption $\La = \{\lam \in\R^n_{++}: A \lam \in S\}$. 
Moreover, we assume that the projection on the set $S$ can be easily computed. 
To this end, we define the $(n+k) \times n$ matrix 
$B\trans = [I,A\trans]$
and the function $\reg(s,t) = \reg_1(s)+\reg_2(t)$, $(s,t) \in \R^{n} \times \R^k$, where 
$$
\reg_1(s) = \frac{\rho}{2} \sum_{i=1}^n \left(\frac{\alpha_i^2}{s_i+\rho} + s_i
+ \delta_{\R_{++}}(s_i) \right),
$$
and $\reg_2(t) = \delta_S(t)$. 
Note that the solution of problem \eqref{eq:sim} is the same as the proximity map of the linearly 
composite function $\reg \circ B$ at $\mu$, which solves the problem
$$
\min \left\{ \frac{1}{2} \|\lam-\mu\|^2 + \reg(B\lambda) : \lambda \in \R^n\right\}.
$$
At first sight this problem seems difficult to solve. However, it turns
out that if the proximity map of the function $\reg$ has a simple form,
the following theorem adapted from \cite[Theorem 3.1]{MSX} can be used
to accomplish this task. For ease of notation we set $d=n+k$.
\begin{theorem}
Let $\reg$ be a convex function on $\R^{d}$, $B$ a $d\times n$ matrix, 
$\mu \in\R^n$, $c >0$,  and define the 
mapping $H: \R^{d} \rightarrow \R^{d}$ at $v \in \R^d$ as
$$H(v)=
(I - \prox_{\frac{\reg}{c}}) ((I-c BB\trans)v + B\mu).
$$ 
Then, for any fixed point ${\hat v}$ of $H$, it holds that
\beq
\prox_{\reg\circ B}(\mu)= \mu - c B\trans {\hat v}.
\label{eq:prox-fixed}
\eeq
\label{thm:fixed}
\end{theorem}

\vspace{-4ex}
The {\em Picard iterates} $\{v_s : s \in \N\}\subseteq \R^d$, starting at $v_0 \in \R^{d}$,
are defined by the recursive equation $v_s = H(v_{s-1})$.  
Since 
the operator $I- \prox_{\reg}$ is {\em nonexpansive}\footnote{A mapping $T: \R^d \rightarrow \R^d$ is called nonexpansive if
$\|T(v)-T(v')\|_2\leq \|v-v'\|_2$, for every $v,v'\in \R^d$.} 
(see e.g. \cite{combettes}), the map $H$ is nonexpansive if $c\in
\left[0,\frac{2}{||B||^2}\right]$. Despite this, the Picard
iterates are not guaranteed to converge to a fixed point of $H$. However, a simple modification
with an averaging scheme can be used to compute the fixed point.

\begin{theorem} \cite{opial}
Let $H: \R^{d} \rightarrow \R^{d}$ be a nonexpansive mapping which has
at least one fixed point and let $H_\ka := \ka I + (1-\ka) H$.
Then, for every $\ka \in (0,1)$, the Picard iterates of $H_\ka$
converge to a fixed point of $H$.
\label{thm:opial}
\end{theorem}

The required proximity operator of $\reg$ is directly given, 
for every $(s,t) \in \R^n \times \R^k$, by 
$$
\prox_\reg(s,t) =
\left(\prox_{\reg_1}(s),\prox_{\reg_2}(t)\right).
$$
Both $\prox_{\reg_1}$ and $\prox_{\reg_2}$ can be easily computed.
The latter requires computing the projection on the set $S$. 
The former requires, for each component of the vector $s \in \R^n$, the
solution of a cubic equation as stated in the next lemma.

\begin{lemma}
For every $\mu,\alpha \in \R$ and $r,\rho > 0$, the function 
$h: \R_+ \rightarrow \R$ defined at $s$ 
as $h(s) := (s-\mu)^2 + r \left(\frac{\alpha^2}{s+\rho} + s\right)$ 
has a unique minimum on its domain, which is attained at $(x_0-\rho)_+$, 
where $x_0$ is the largest real root of the polynomial 
$2x^3 +(r-2(\mu+\rho))x^2 - r \alpha^2$.
\vspace{-1ex}
\label{lem:cubic}
\end{lemma}
\begin{proof}
Setting the derivative of $h$ equal to zero and making the 
change of variable $x=s+\rho$ yields the polynomial stated 
in the lemma. Let $x_0$ be the largest root of this polynomial.
Since the function $h$ is strictly convex on its domain and grows at infinity, 
its minimum can be attained only at one point, which is $x_0-\rho$, if $x_0>\rho$, and zero otherwise.
\end{proof}

\subsection{Accelerated Proximal Method}
\label{sec:fista}

Theorem \ref{thm:fixed} motivates a proximal numerical approach to solving problem \eqref{eq:primal} and, in turn, problem \eqref{eq:method}.  Let
$E(\beta) = \frac{1}{2}\|X\beta-y\|^2_2$ and assume an upper bound $L$
of $\|X\trans X\|$ is known\footnote{For variants of such algorithms
which adaptively learn $L$, see the above references.}.  Proximal
first-order methods -- see \cite{combettes,fista,Nesterov07,tseng08} and
references therein -- can be used for nonsmooth optimization, where
the objective consists of a strongly smooth term, plus a nonsmooth
part, in our case $E$ and $\Gamma+\delta_\Lambda$, respectively. The
idea is to replace $E$ with its linear approximation around a point
$w_t$ specific to iteration $t$.  
This leads to the computation of a proximity operator, and specifically in our case
to $u_t := (\beta_t, \lambda_t) \leftarrow 
\argmin \bigl\{ \frac{L}{2} \|(\beta,\lam) - (w_t -  \frac{1}{L}\nabla E(w_t))\|^2_2 + 
\rho \,\Gamma(\beta,\lam) : \beta \in \R^n, \lam \in \Lambda \bigr\}$. 
Subsequently, the point $w_t$ is updated, based on the
current and previous estimates of the solution $u_t, u_{t-1}, \dots$
and the process repeats.\vspace{-0.5ex}

\vspace{-0.5ex}The simplest update rule is $w_{t}=u_{t}$.
By contrast, {\em accelerated proximal methods} proposed by
\cite{Nesterov07} use a carefully chosen $w$ update with two levels of
memory, $u_t, u_{t-1}$.
If the proximity map can be exactly computed, such schemes 
exhibit a fast quadratic decay in terms of the
iteration count, that is, the distance of the objective from the
minimal value is $O\left(\frac{1}{T^2}\right)$ after $T$ iterations.
However, it is not known whether accelerated methods 
which compute the proximity operator {\em numerically} can achieve this rate.
The main advantages of accelerated methods are their 
low cost per iteration and their scalability to large problem sizes.
Moreover, in applications where a thresholding operation
is involved -- as in Lemma \ref{lem:cubic} --
the zeros in the solution are exact.

\begin{algorithm}
\caption{NEsterov PIcard-Opial algorithm (NEPIO)}
\begin{algorithmic}
\STATE $u_1,w_1 \leftarrow$ arbitrary feasible values
\FOR {t=1,2,\dots}
\STATE Compute a fixed point ${\hat v^{(t)}}$ of $H_t$ by Picard-Opial 
\STATE $u_{t+1} \leftarrow w_t - \frac{1}{L}\nabla E(w_t)
-\frac{c}{L} B\trans {\hat v^{(t)}}$
\STATE $w_{t+1} \leftarrow \pi_{t+1} u_{t+1} - (\pi_{t+1}-1)u_t$
\ENDFOR 
\end{algorithmic}
\label{alg:acc}
\end{algorithm}

For our purposes, we use a version of accelerated methods
influenced by \cite{tseng08} (described in Algorithm \ref{alg:acc} and called {\em NEPIO}). 
According to Nesterov \cite{Nesterov07}, the optimal update is 
$w_{t+1} \leftarrow u_{t+1} + \theta_{t+1}\left(\frac{1}{\theta_t}-1\right) 
(u_{t+1}-u_t)$
where the sequence $\theta_t$ is defined by $\theta_1 = 1$ 
and the recursion
\beq
\frac{1-\theta_{t+1}}{\theta_{t+1}^2} = \frac{1}{\theta_t^2} \,.
\label{eq:theta}
\eeq
We have adapted \cite[Algorithm 2]{tseng08} 
(equivalent to FISTA \cite{fista}) by computing the proximity operator of
$\frac{\reg}{L}\circ B$ using the Picard-Opial process described in
Section \ref{sec:prox_map}.
We rephrased the algorithm using the sequence 
$\pi_t := 1-\theta_t + \sqrt{1-\theta_t}
= 1-\theta_t +\frac{\theta_t}{\theta_{t-1}}$ for numerical 
stability.
At each iteration, the map $H_t$ is defined by 
$$
H_t(v) := \left(I-\prox_{\frac{\phi}{c}}\right) \Bigg(
\left(I-\frac{c}{L} BB\trans\right) v - \frac{1}{L}B \Big(\nabla E(w_t) 
- Lw_t\Big) \Bigg)\,.
$$
We also apply an Opial averaging so that the
update at stage $s$ of the proximity computation 
is $v_{s+1} = \kappa v_s + (1-\kappa) H_t(v_s)$.
By Theorem \ref{thm:fixed}, the fixed
point process combined with the assignment of $u$ are equivalent to $u_{t+1}
\leftarrow \prox_{\frac{\reg}{L}\circ B}\left(w_t -
  \frac{1}{L}\nabla E(w_t)\right)$.

The reason for resorting to Picard-Opial is that
exact computation of the proximity operator
\eqref{eq:sim} is possible only in simple
special cases for the set $\Lambda$.
By contrast, our approach can be applied to a
wide variety of constraints.
Moreover, we are not aware of another proximal
method for solving problems \eqref{eq:primal} or 
\eqref{eq:method} 
and alternatives like interior point methods
do not scale well with problem size.
In the next section, we will demonstrate
empirically the scalability of Algorithm \ref{alg:acc},
as well as the efficiency of both the proximity
map computation and the overall method.

\section{Numerical Simulations}
\label{sec:5}

In this section, we present experiments with method \eqref{eq:primal}. 
The goal of the experiments is to both study the computational and the
statistical estimation properties of this method. 
One important aim of the experiments is to demonstrate that the method is statistically competitive 
or superior to state-of-the-art methods while being computationally efficient. 
The methods employed are the Lasso, StructOMP \cite{huang2009} and method \eqref{eq:primal} 
with the following choices for the constraint set $\Lambda$: {\em Grid-C}, $\La_\alpha = \{\lam:\|A\lam\|_1\leq \alpha\}$, where $A$ is the edge map of  
a 1D or 2D grid and $\alpha > 0$, and {\em Tree-C}, $\La=\{\lam: A\lam \geq 0\}$, where $A$ is the edge map of a tree graph.

We solved the optimization problem \eqref{eq:primal} either with the toolbox CVX\footnote{\tt http://cvxr.com/cvx/} or with the proximal method presented in Section
\ref{sec:prox}. 
When using the proximal method, we chose the parameter from Opial's Theorem $\kappa=0.2$ and we stopped the iterations when the relative decrease in the  objective value is less than $10^{-8}$. 
For the computation of the proximity operator, we stopped the iterations of the Picard-Opial method when the relative difference between two consecutive iterates is smaller than $10^{-2}$. 
We studied the effect of varying this tolerance in the next experiments.
We used  the square loss and computed the Lipschitz constant $L$ using singular value decomposition. 

\begin{figure}[!ht]
\begin{center}
  \begin{tabular}{ccc}
\includegraphics[width=0.32\textwidth]{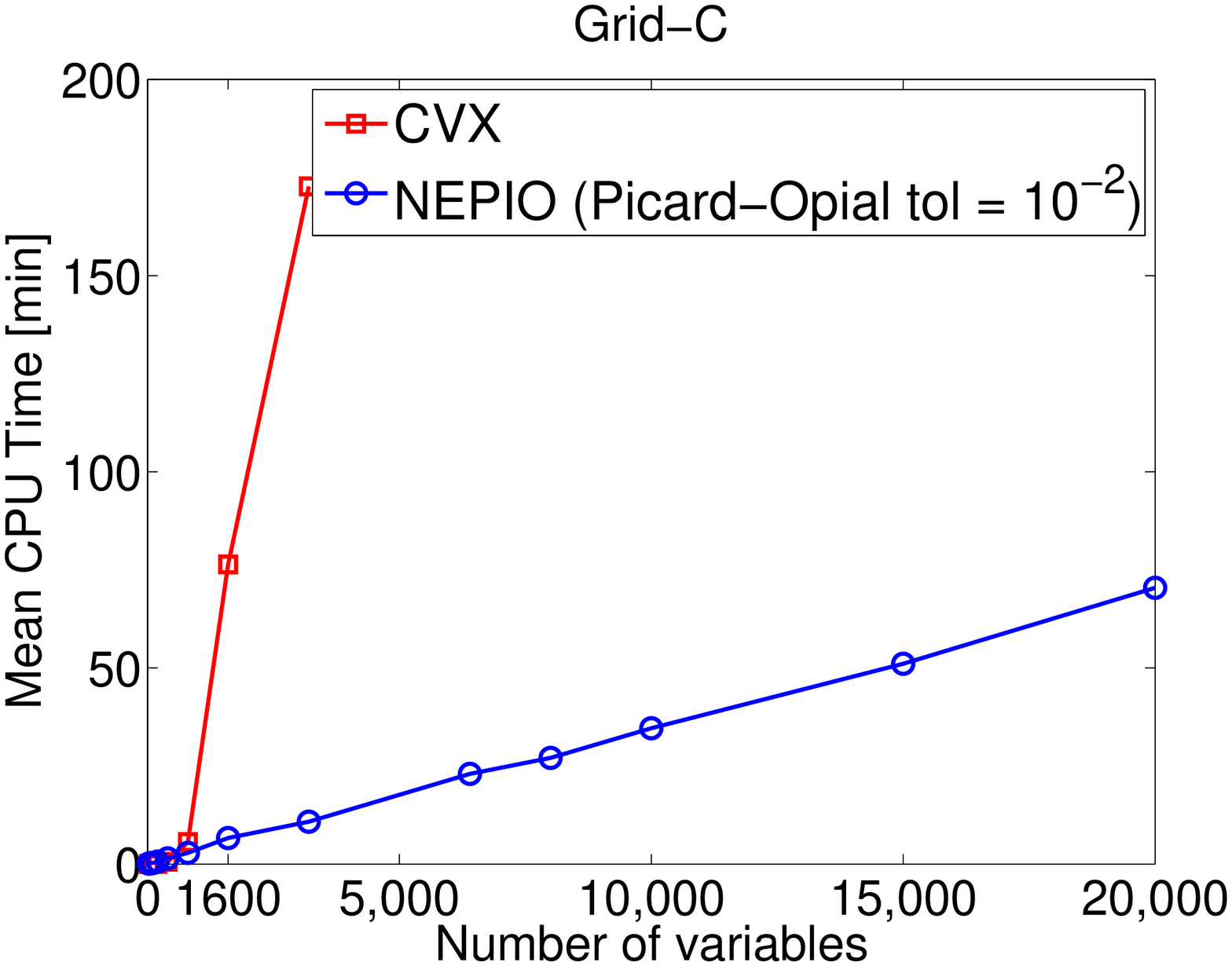} &
\includegraphics[width=0.32\textwidth]{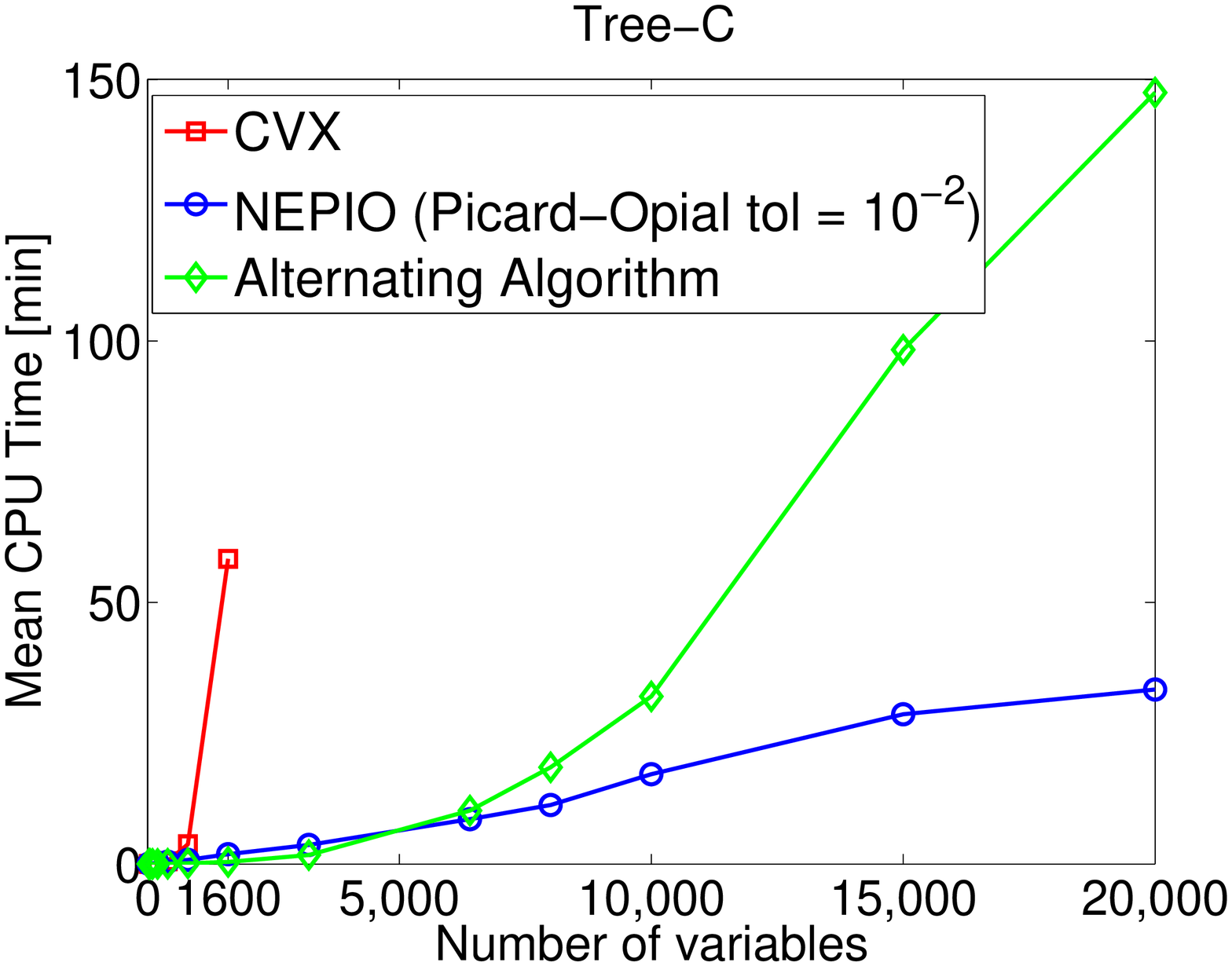} &
\includegraphics[width=0.32\textwidth]{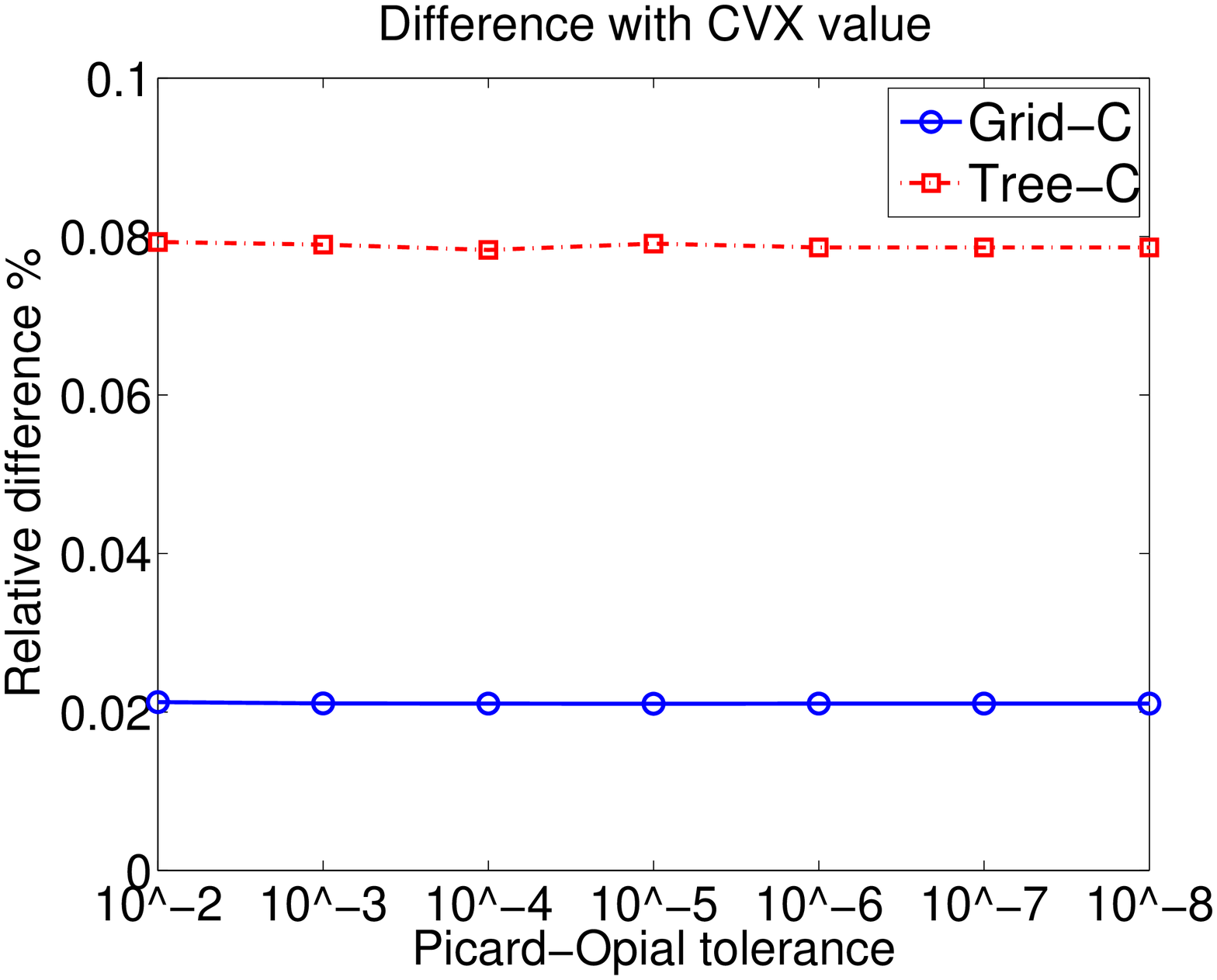}

\\
  \end{tabular}
  \caption{(Left and Centre) Computation time vs problem size for Grid-C and Tree-C. (Right) Difference with the solution obtained via CVX vs Picard-Opial tolerance.}
  \label{fig:cost}
\end{center}
\end{figure}

\subsection{Efficiency experiments}
First, we investigated the computational properties of the proximal method.  
Our aim in these experiments was to show that our algorithm has a time
complexity that scales linearly with the number of variables, while
the sparsity and relative number of training examples is kept constant. We
considered both the Grid and the Tree constraints and compared our
algorithm to the toolbox CVX, which is an interior-point method
solver. As is commonly known, interior-point methods are very fast,
but do not scale well with the problem size. In the case of the Tree
constraint, we also compared with a modified version of the
alternating algorithm of \cite{MMP-10}. For each problem size, we
repeated the experiments $10$ times and we report the average
computation time in Figure \ref{fig:cost}-{\em(Left and Centre)} for Grid-C
and Tree-C, respectively. This result indicates that our method is
suitable for large scale experiments.

We also studied the importance of the Picard-Opial tolerance for converging to a good solution. In Figure \ref{fig:cost}-{\em Right}, we report the average relative distance to the solution obtained via CVX for different values of the Picard-Opial tolerance. We considered a problem with $100$ variables and repeated the experiment $10$ times with different sampling of training examples, considering both the Grid and the Tree constraint.
It is clear that decreasing the tolerance did not bring any advantage in terms of converging to a better solution, while it remarkably increased the computational overhead, passing from an average of $5s$ for a tolerance of $10^{-2}$ to $40s$ for $10^{-8}$ in the case of the Grid constraint.

Finally, we considered the 2D Grid-C case and observed that the number of Picard-Opial iterations needed to reach a tolerance of $10^{-2}$, scales well with the number of variables $n$. 
For example when $n$ varies between $200$ and $6400$, the average number of iterations varied between $20$ and $40$.




\subsection{Statistical experiments}

\noindent {\bf One dimensional contiguous regions.} In this experiment, we chose a model vector $\beta^\ast\in \R^{200}$ with $20$ nonzero elements, whose values are random $\pm 1$. We considered sparsity patterns forming one, two, three or four 
non-overlapping contiguous regions, which have lengths of $20$, $10$, $7$ or $5$, respectively. 
We generated a noiseless output from a matrix $X$
whose elements have a standard Gaussian distribution.
 The estimates $\hat{\beta}$ for several models are then
compared with the original. 
Figure \ref{fig:1} shows the model error $\frac{\lVert\ \hat{\beta} -
\beta^\ast \lVert_2}{\lVert\ \beta^\ast \lVert_2}$ as the 
sample size changes from $22$ (barely above the sparsity) up to $100$
(half the dimensionality). This is the average over $50$ runs, each
with a different $\beta^\ast$ and $X$. 
We observe that Grid-C outperforms both Lasso and StructOMP, whose performance deteriorates as the number of regions is increased.
For one particular run with a sample size which is twice the model sparsity, Figure \ref{fig:2} shows the original vector and the estimates for different methods. 

\begin{figure}
\begin{center}
  \begin{tabular}{cc}
     \includegraphics[width=0.233\textwidth]{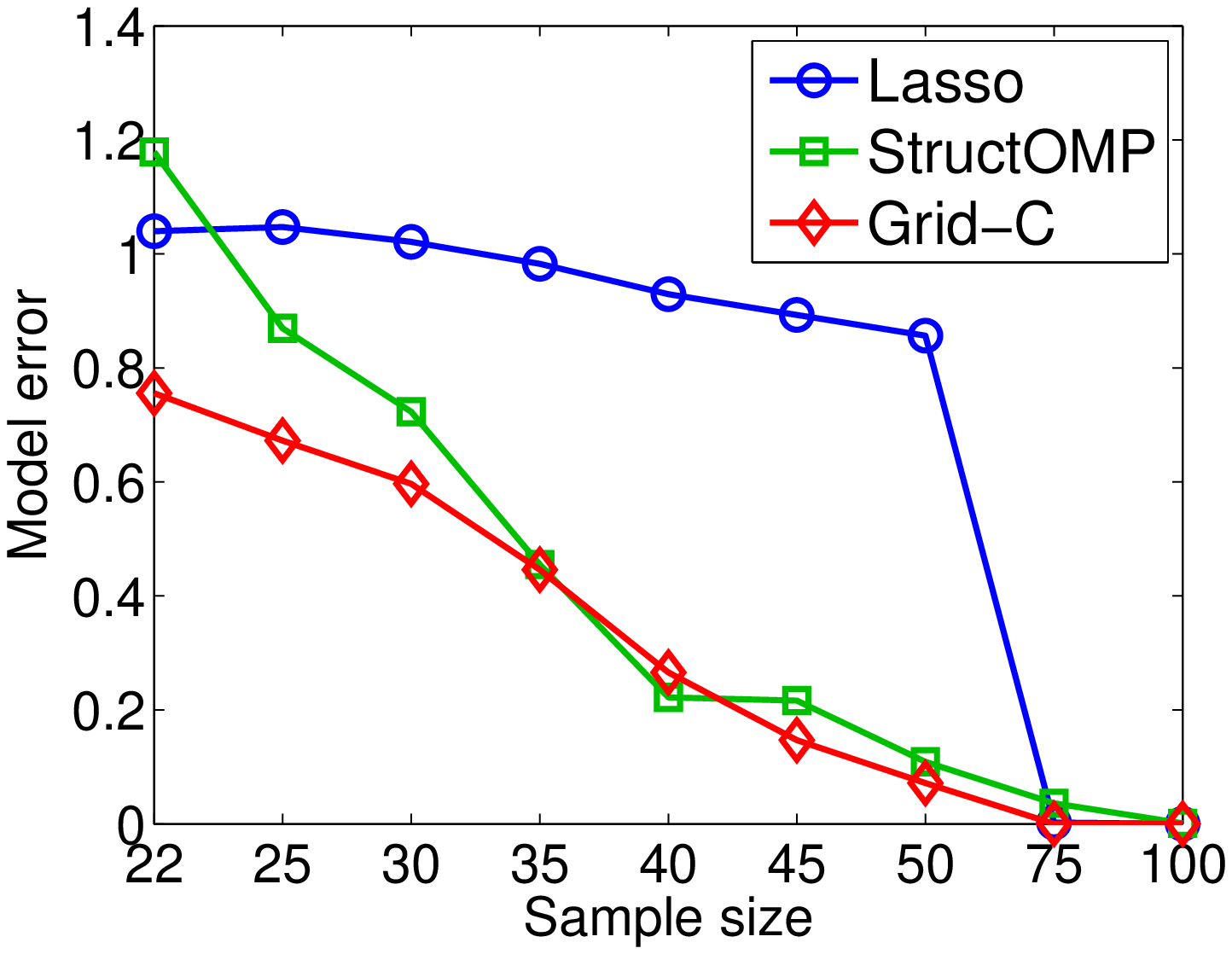} 
     \includegraphics[width=0.233\textwidth]{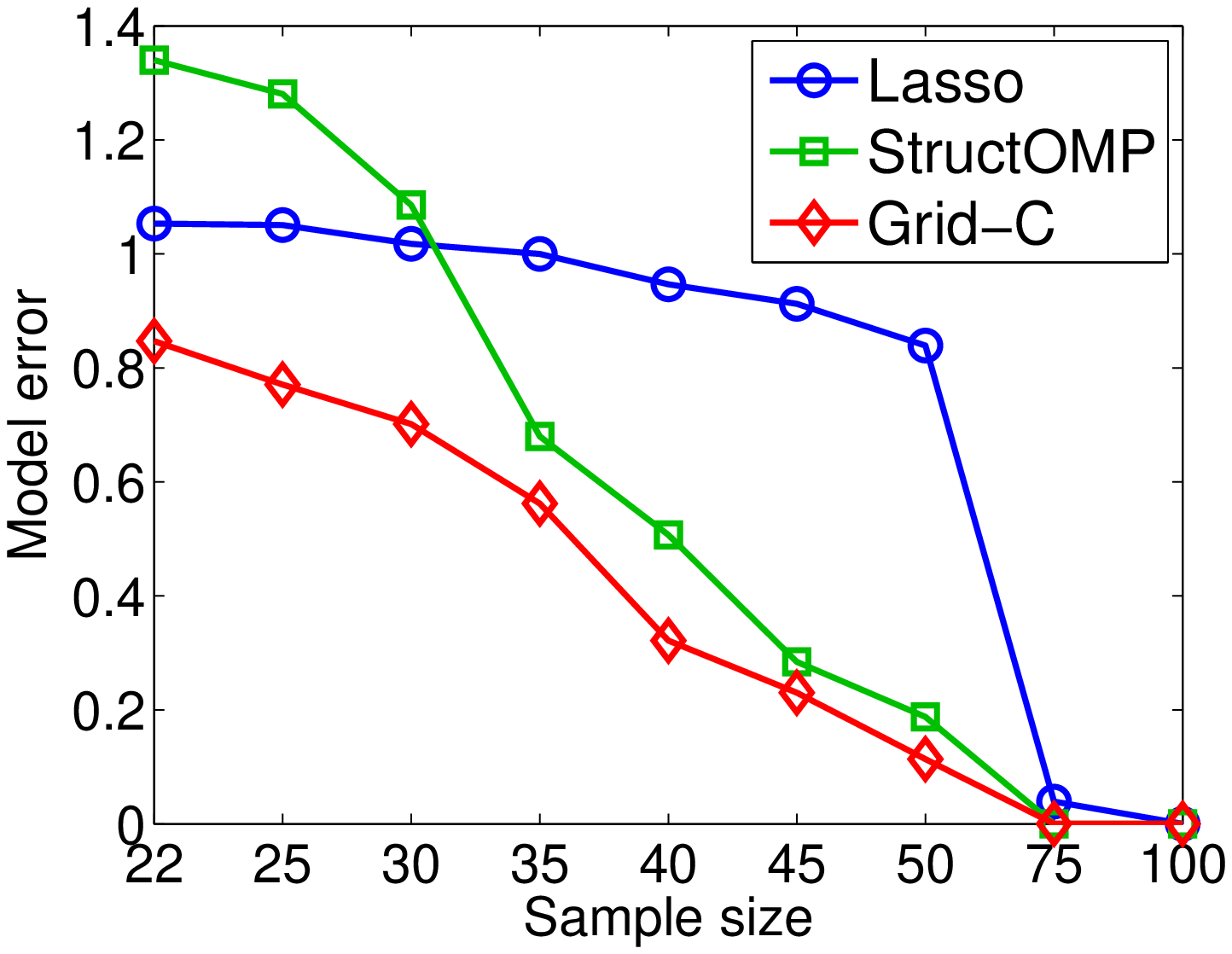} 
     \includegraphics[width=0.233\textwidth]{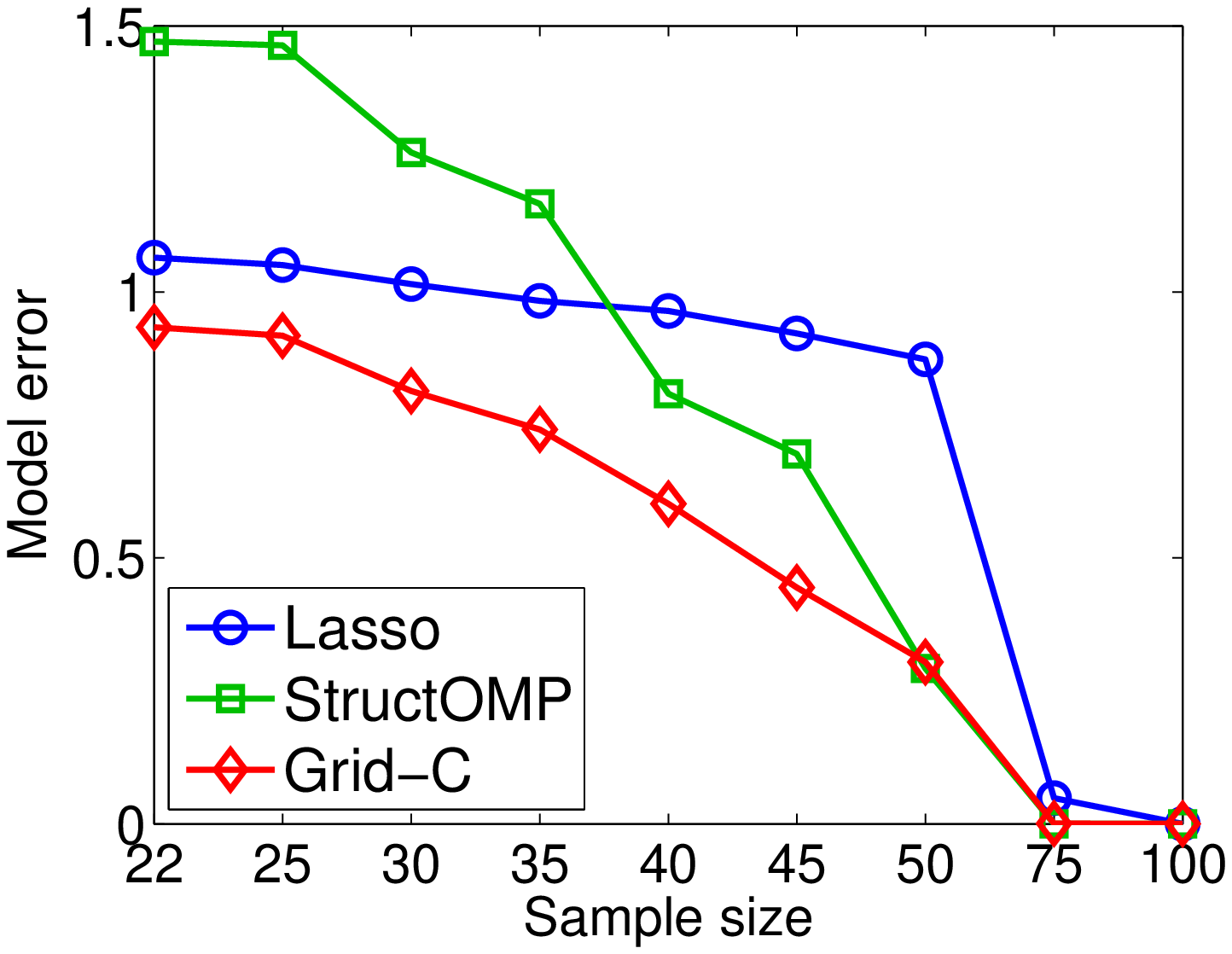} 
     \includegraphics[width=0.233\textwidth]{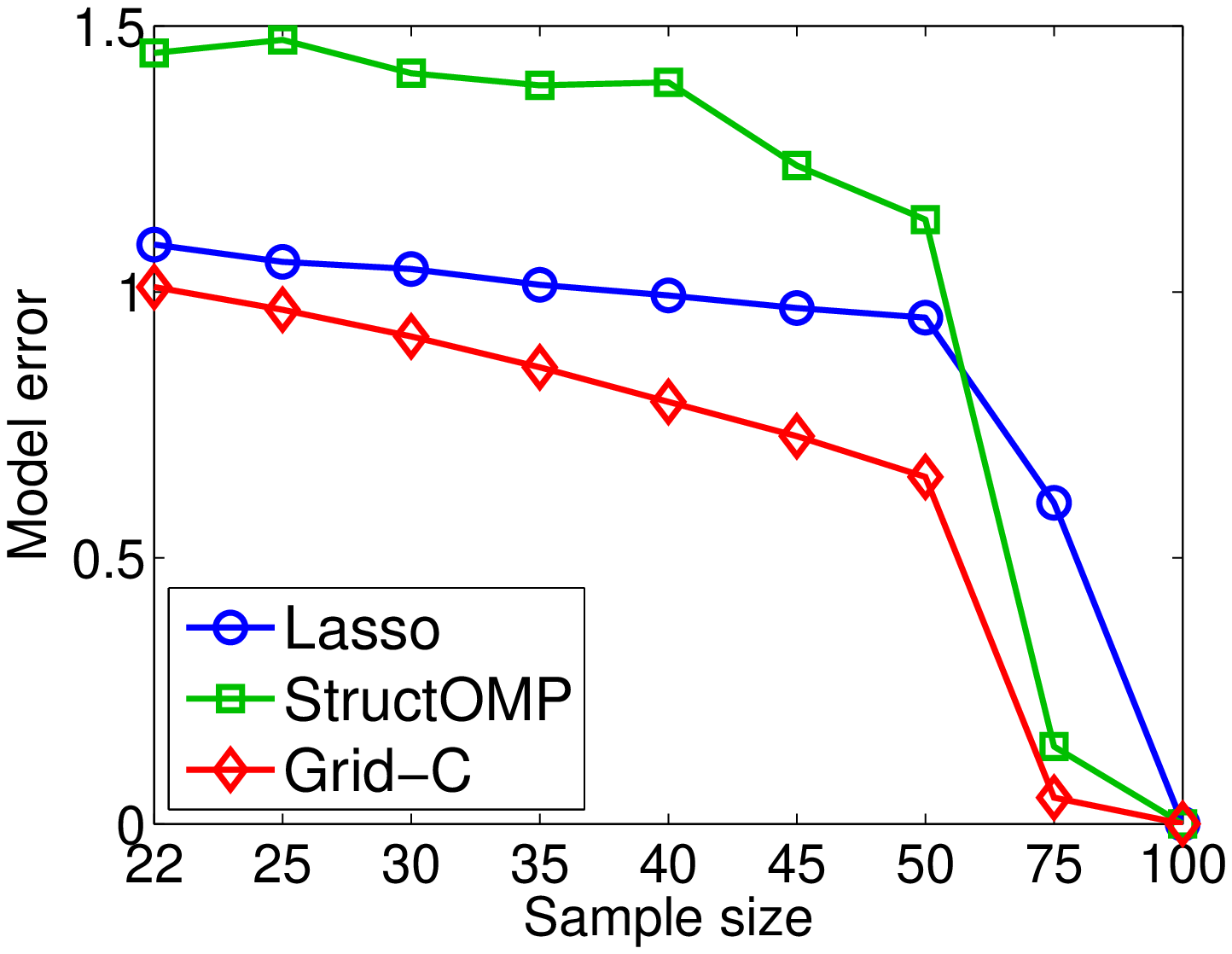} 
  \end{tabular}
  \caption {$1D$ contiguous regions: comparison between different
    methods for one (left), two (centre-left), three (centre-right) and four (right) regions.}
  \label{fig:1}
\end{center}
\end{figure}

\noindent {\bf Two dimensional contiguous regions.} 
We repeated the experiment in the case that the sparsity
pattern of $\beta^*\in \R^{20\times20}$ consists of $2D$ rectangular regions. 
We considered either a single $5\times5$ region,
two regions ($4\times4$ and $3\times3$), three $3\times3$ regions and
four $3\times2$ regions.  Figure \ref{fig:3} shows the model error versus
the sample size in this case. Figure \ref{fig:4} shows the original image 
and the images estimated by different methods for a sample size which
is twice the model sparsity.  Note that Grid-C is superior to both the
Lasso and StructOMP and that StructOMP is outperformed by Lasso when
the number of regions is more than two. This behavior 
is consistently confirmed by experiments in
higher dimensions, not shown here for brevity.

\begin{figure}
\begin{center}
  \begin{tabular}{cccc}
\includegraphics[width=0.23\textwidth,height = 1.8cm]{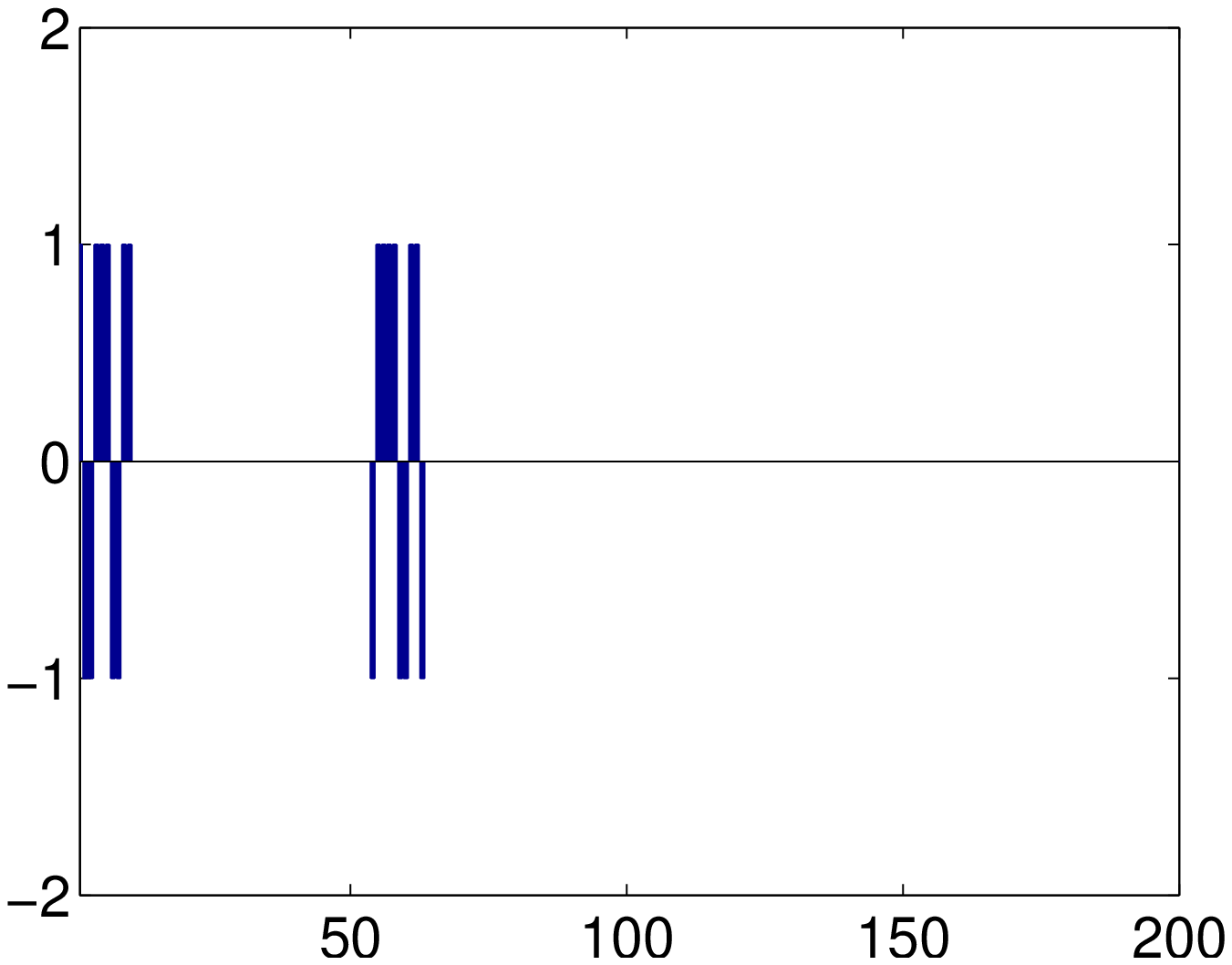} &
\includegraphics[width=0.23\textwidth,height = 1.8cm]{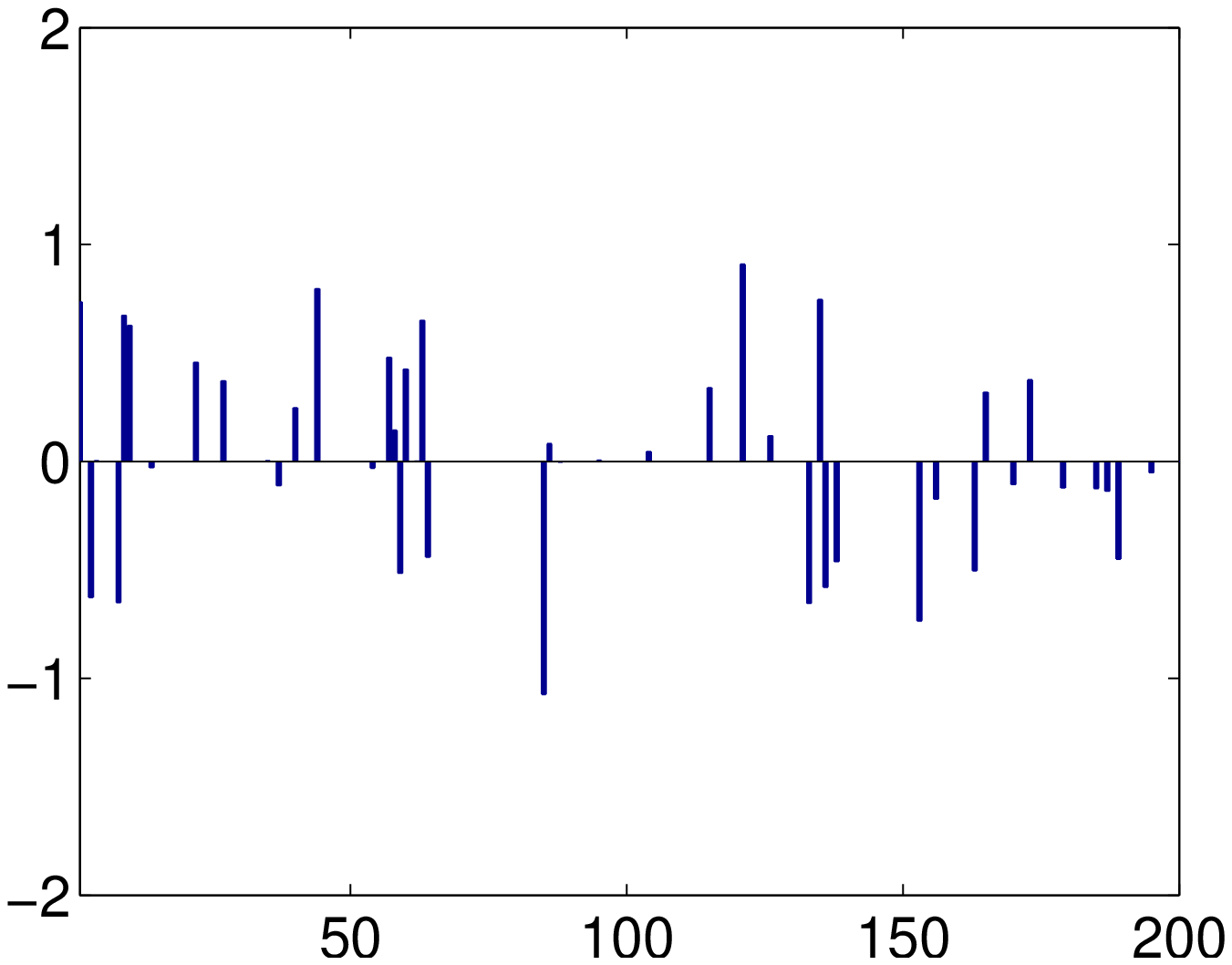} &
\includegraphics[width=0.23\textwidth,height = 1.8cm]{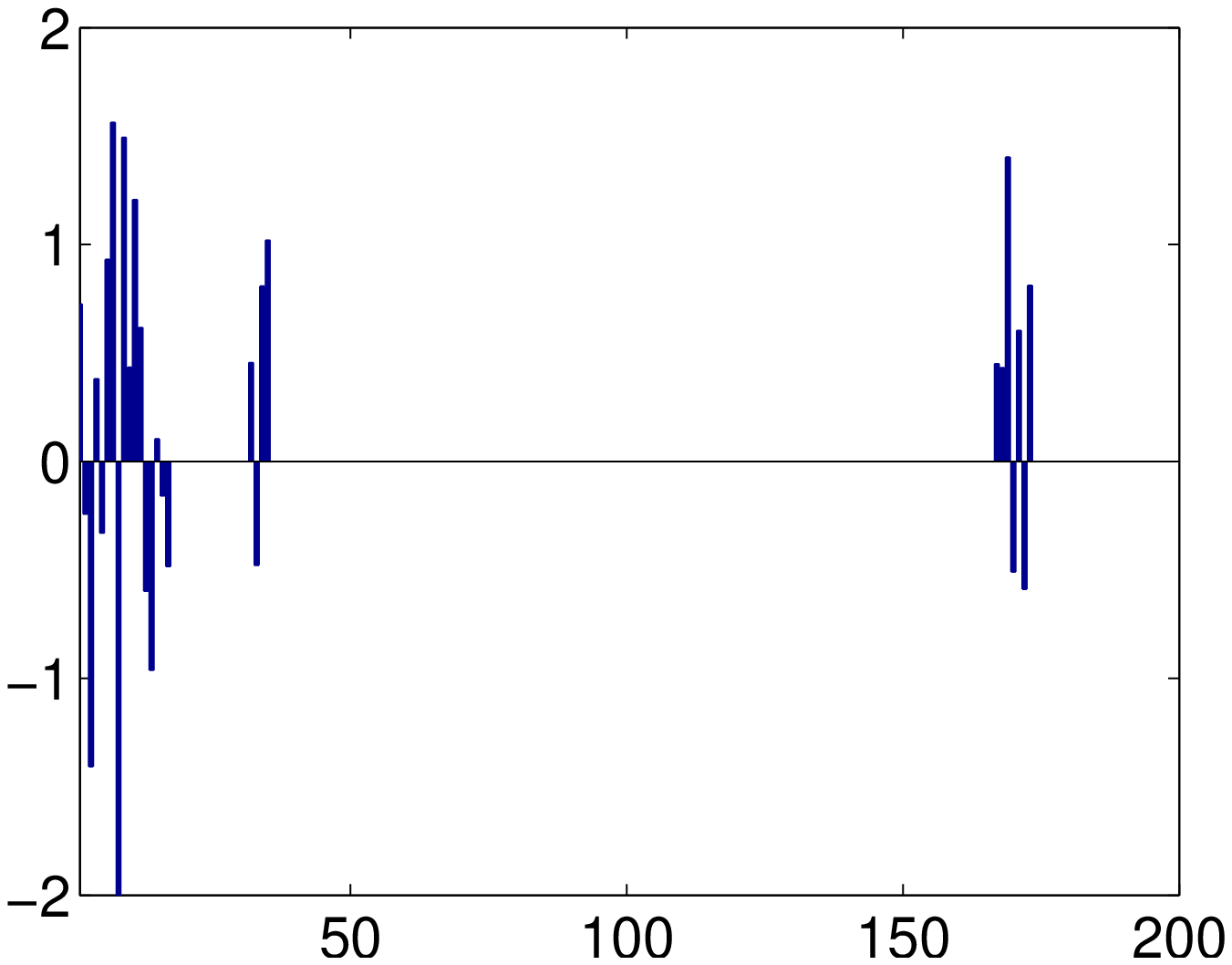} &
\includegraphics[width=0.23\textwidth,height = 1.8cm]{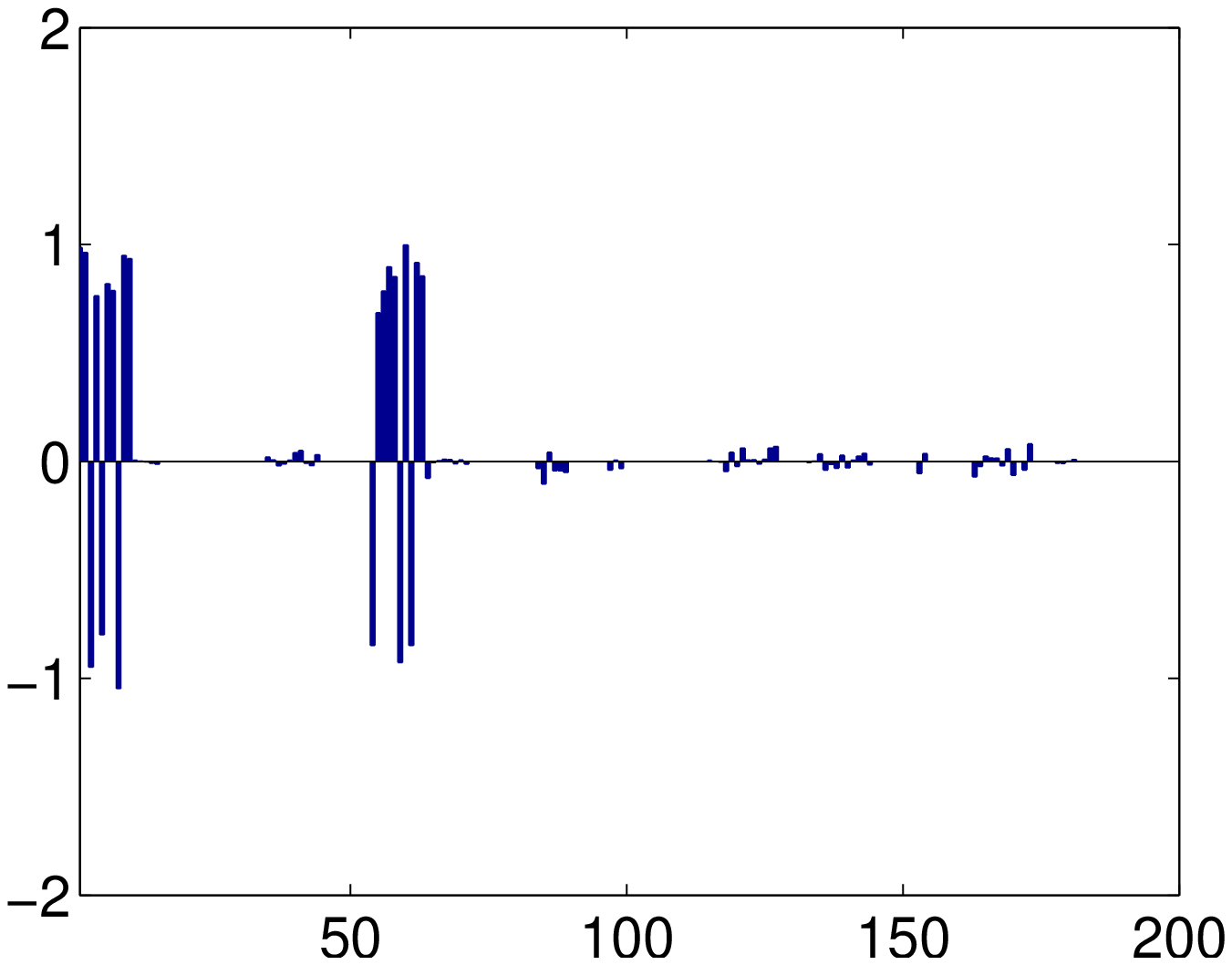}
  \end{tabular}
  \caption {Two $1D$ contiguous regions: regression vector estimated
    by different models: $\beta^*$ (left), Lasso (centre-left),
    StructOMP (centre-right), Grid-C (right).}
  \label{fig:2}
\end{center}
\end{figure}

\begin{figure}
\begin{center}

  \begin{tabular}{cc}
     \includegraphics[width=0.233\textwidth]{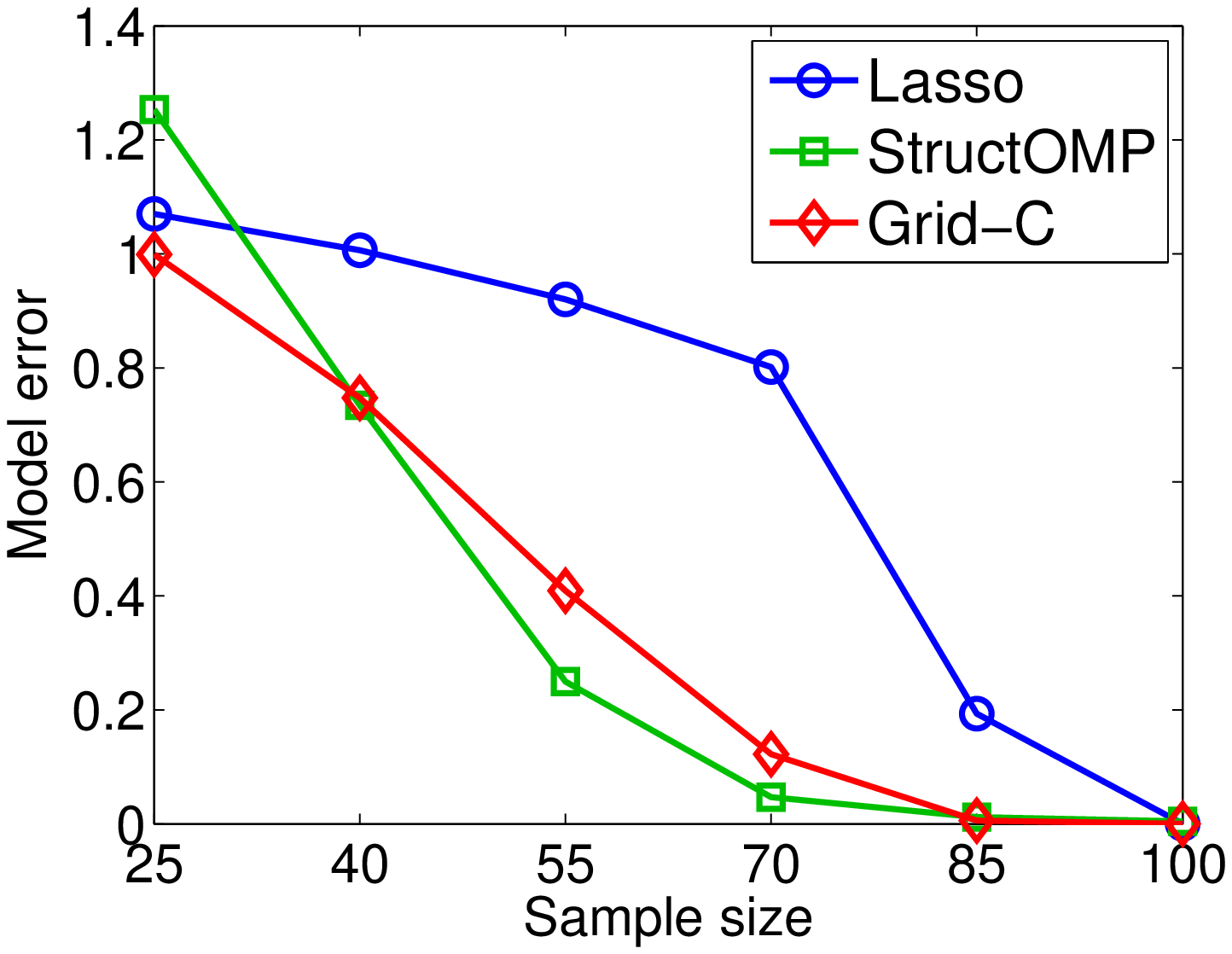} 
     \includegraphics[width=0.233\textwidth]{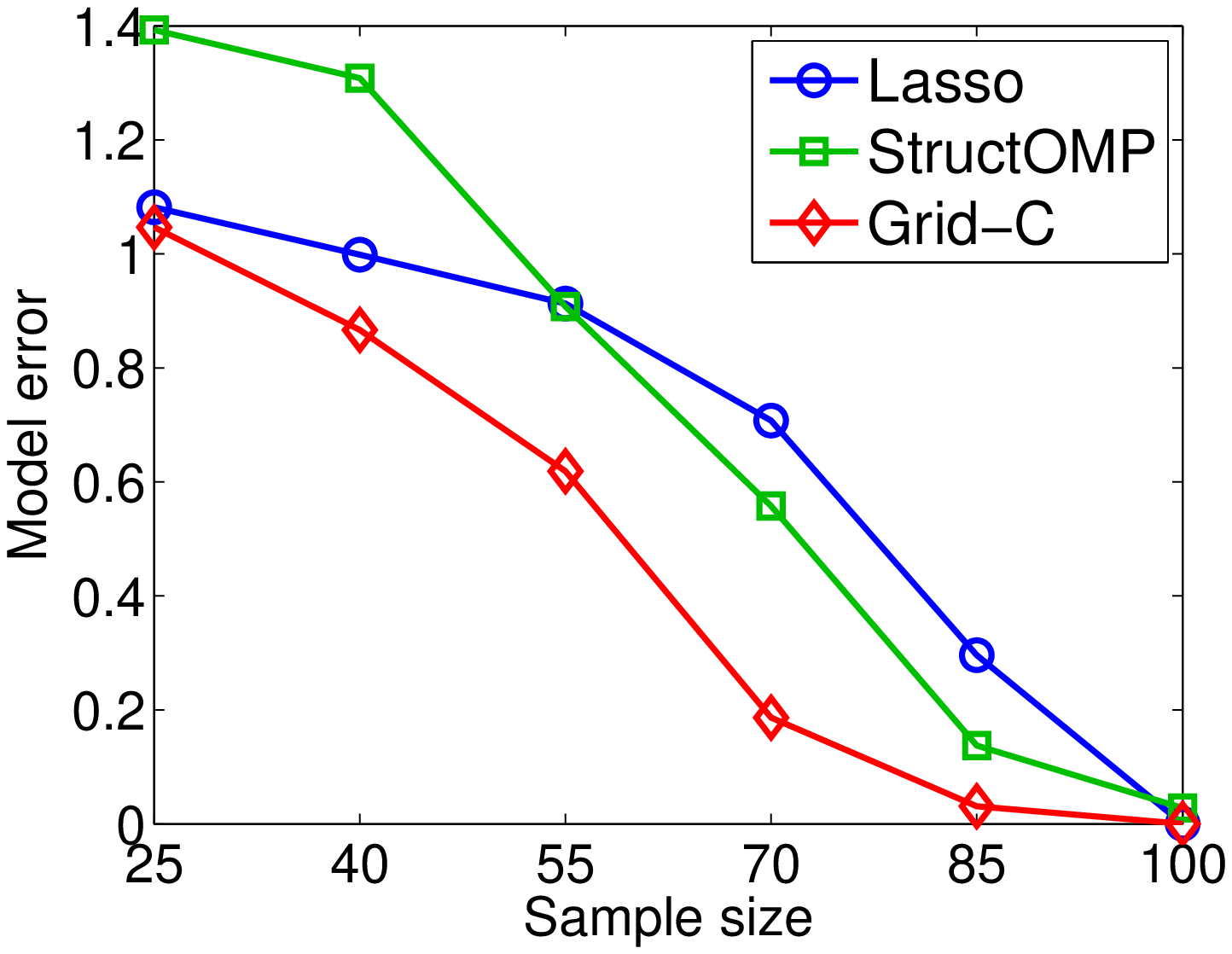} 
     \includegraphics[width=0.233\textwidth]{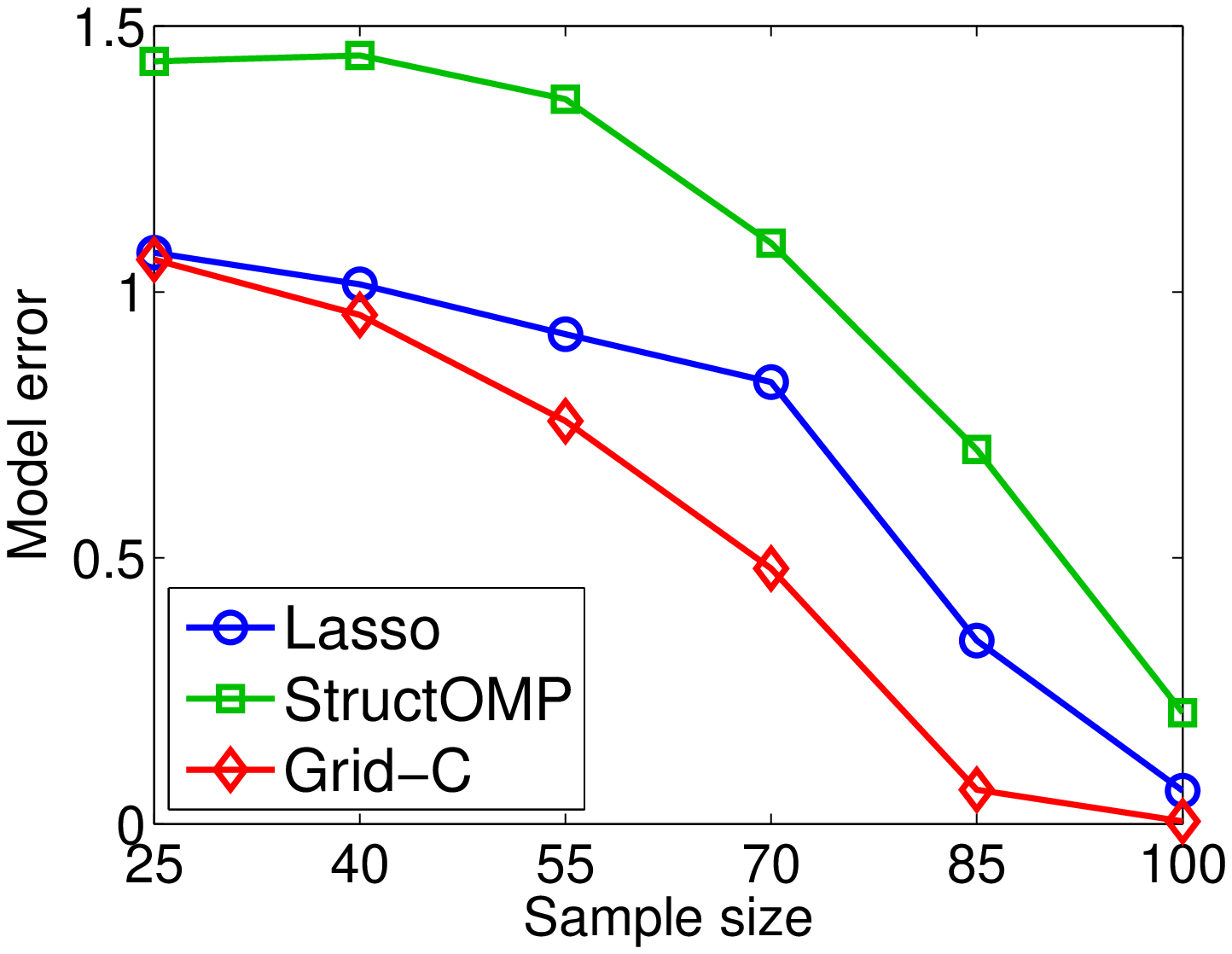} 
     \includegraphics[width=0.233\textwidth]{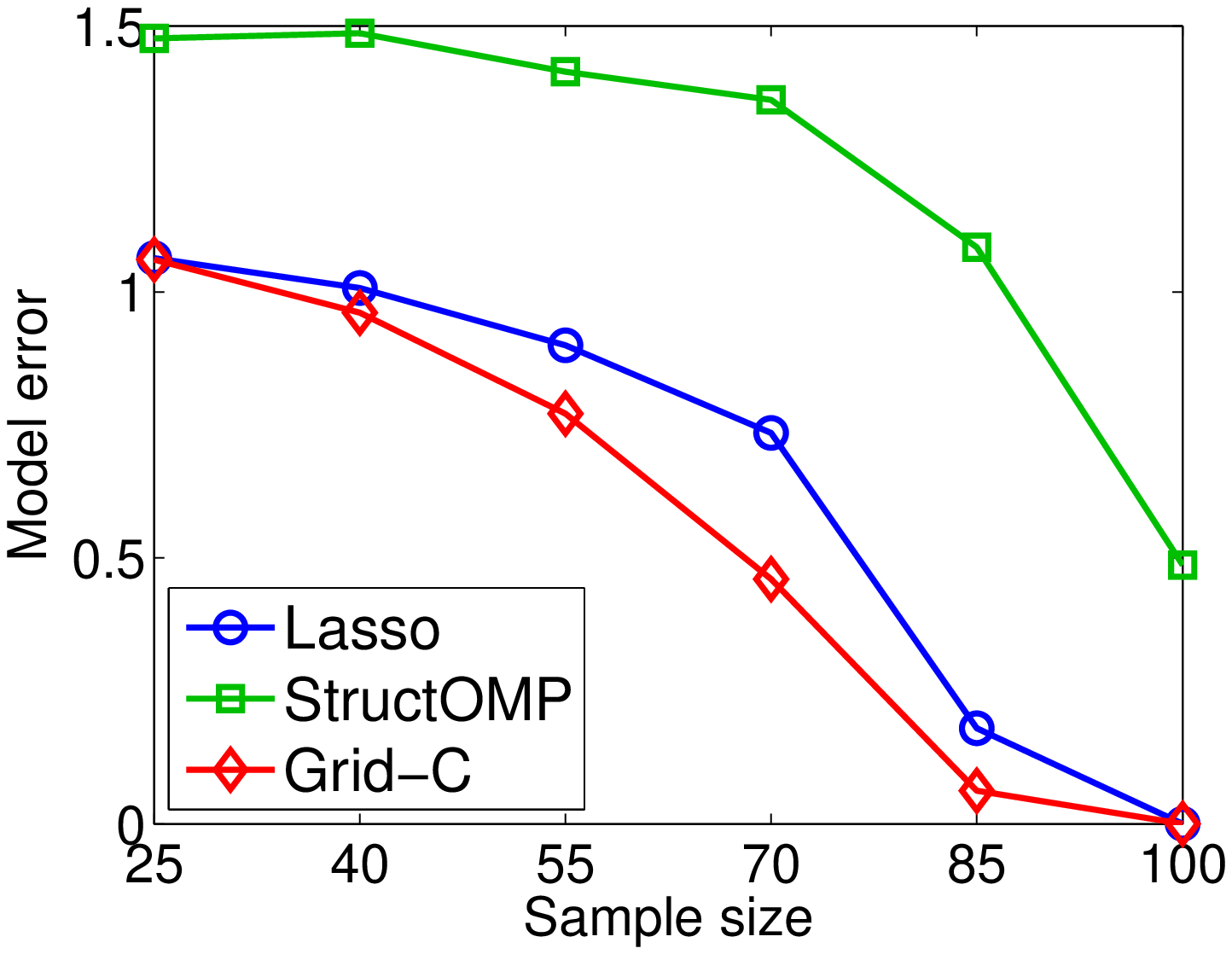} 
  \end{tabular}
  \caption {$2D$ contiguous regions: comparison between different
    methods for one (left), two (centre-left), three (centre-right) 
    and four (right) regions.}
  \label{fig:3}
\end{center}
\end{figure}

\begin{figure}
\begin{center}
  \begin{tabular}{cc}
  \includegraphics[height=1.4cm,width=1\textwidth]{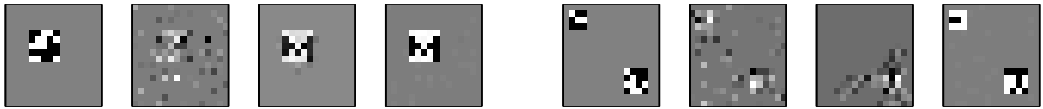}
  \end{tabular}
  \caption {$2D$-contiguous regions: model vector and vectors
    estimated by the Lasso, StructOMP and Grid-C (left to right), for
    one region    (first group) and two regions (second group).}
  \label{fig:4}
\end{center}
\end{figure}

\noindent {\bf Background subtraction.} We replicated the experiment from \cite[Sec.~7.3]{huang2009} with our method. Briefly, 
the underlying model $\beta^*$ corresponds to the pixels of the foreground of a CCTV image. 
We measured the output as a random projection plus Gaussian noise.  
Figure \ref{fig:back}-{\em Left} shows that, while the Grid-C outperforms the Lasso, it is not 
as good as StructOMP. We speculate that this result is
due to the non uniformity of the values of the image, which makes it harder for Grid-C to estimate the model. 

\noindent {\bf Image Compressive Sensing.} In this experiment, we compared the performance of Tree-C on an instance of 2D image compressive sensing, following the
experimental protocol of \cite{huang2009}. Natural images can be well
represented with a wavelet basis and their wavelet coefficients,
besides being sparse, are also structured as a hierarchical tree. We
computed the Haar-wavelet coefficients of a widely used {\em cameraman} image.
We measured the output as a random projection plus Gaussian noise.
StructOMP and Tree-C, both exploiting the tree structure, were used to recover the wavelet coefficients from the measurements and compared to the Lasso. The inverse wavelet transform was used to
reconstruct the images with the estimated coefficients. The
recovery performances of the methods are reported in Figure
\ref{fig:back}-{\em Right}, which highlights the good performance of
Tree-C.

\begin{figure}[h!]
\begin{center}
  \begin{tabular}{cc}
     \includegraphics[width=0.3\textwidth]{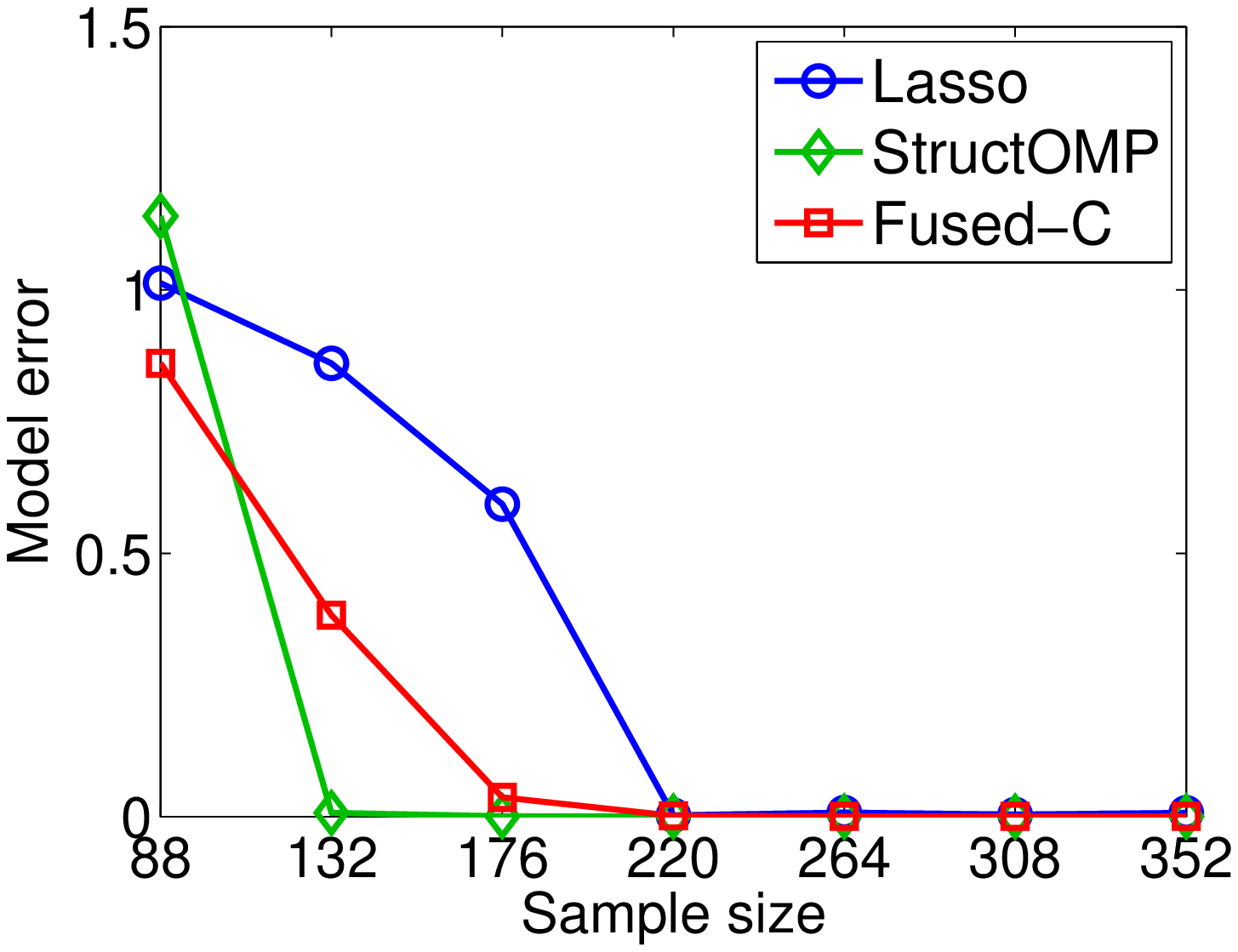}
     \includegraphics[width=0.3\textwidth]{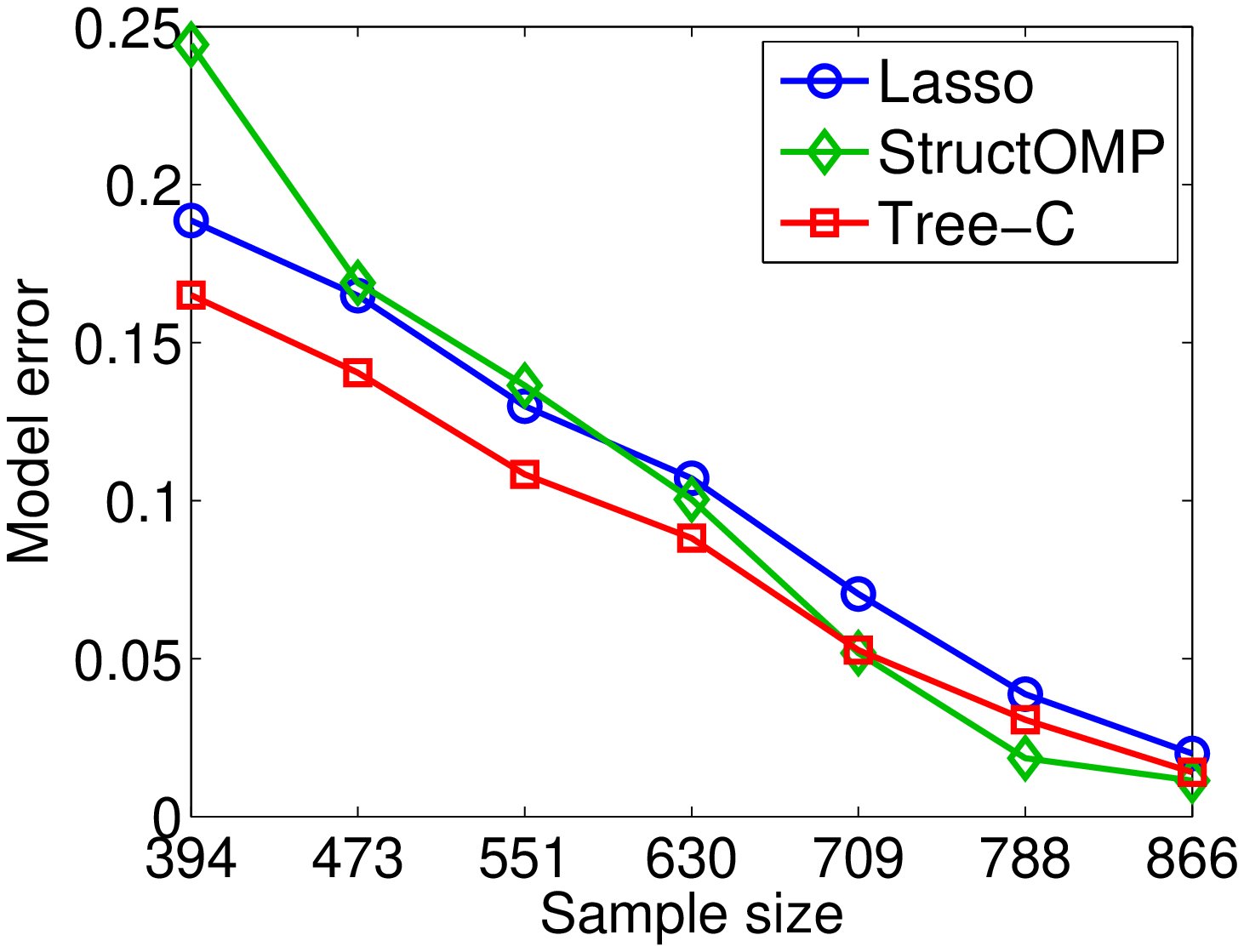}\\
  \end{tabular}
  \caption {Model error for the background subtraction (left) and {\em cameraman} (right) experiments.}
  \label{fig:back}
\end{center}
\end{figure}


\section{Conclusion}
\label{sec:6}

We proposed new families of penalties and presented a new algorithm and results on the class of structured sparsity penalty functions proposed by \cite{MMP-10}. 
These penalties can be used, among else, to learn regression vectors whose sparsity pattern is formed by few contiguous regions. 
We presented a proximal method for solving this class of penalty functions and derived an efficient fixed-point method for computing the proximity operator of our penalty.
We reported encouraging experimental results, which
highlight the advantages of the proposed penalty function over a state-of-the-art greedy method \cite{huang2009}. At the same time, our
numerical simulations indicate that the proximal method is
computationally efficient, scaling linearly with the problem size.
An important problem which we wish to address in the future is to study the
convergence rate of the method and determine whether the optimal rate
$O(\frac{1}{T^2})$ can be attained.  Finally, it would be important to
derive sparse oracle inequalities for the estimators studied here.

\bibliographystyle{plain}


\end{document}